\documentclass[runningheads]{llncs}
\usepackage{graphicx}
\usepackage{amsmath,amssymb} 
\usepackage{color}\usepackage[width=122mm,left=12mm,paperwidth=146mm,height=193mm,top=12mm,paperheight=217mm]{geometry}
\usepackage{url,xspace,caption,subfigure,float}
\usepackage{array,multirow,booktabs,balance}
\usepackage{times,txfonts}
\usepackage{mathtools}
\usepackage[ruled,vlined]{algorithm2e}
\usepackage{algorithmic}
\usepackage{hyperref} 
\hypersetup{backref,colorlinks=true,linkcolor=red,citecolor=green,urlcolor=blue}

\makeatletter
\DeclareRobustCommand\onedot{\futurelet\@let@token\@onedot}
\def\@onedot{\ifx\@let@token.\else.\null\fi\xspace}

\def\eg{\emph{e.g}\onedot} 
\def\ie{\emph{i.e}\onedot}

\def\etal{\emph{et al}\onedot}
\makeatother

\def\Vec#1{{\boldsymbol{#1}}}
\def\Mat#1{{\boldsymbol{#1}}}

\def\SPD#1{\mathcal{S}_{++}^{#1}}

\newcommand{\tr}{\mathop{\rm  Tr}\nolimits}
\newcommand{\DIAG}{\mbox{Diag\@\xspace}}

\begin{document}
\pagestyle{headings}
\mainmatter
\title{From Manifold to Manifold: Geometry-Aware Dimensionality Reduction for SPD Matrices} 

\titlerunning{Geometry-Aware Dimensionality Reduction for SPD Matrices}

\authorrunning{Harandi \etal}

\author
  {
  Mehrtash~T.~Harandi
  \and
  Mathieu~Salzmann
  \and
  Richard~Hartley
  }

\institute
  {
  Australian National University, Canberra, ACT 0200, Australia\\
  ~NICTA\thanks
    {
    NICTA is funded by the Australian Government as represented by the Department of
    Broadband, Communications and the Digital Economy, as well as by the Australian
    Research Council through the ICT Centre of Excellence program.
    }%
    , Locked Bag 8001, Canberra, ACT 2601, Australia%
  }

\maketitle

\begin{abstract}
Representing images and videos with Symmetric Positive Definite (SPD) matrices and considering the Riemannian geometry of the resulting space has proven beneficial for many recognition tasks. Unfortunately, computation on the Riemannian manifold of SPD matrices --especially of high-dimensional ones-- comes at a high cost that limits the applicability of existing techniques.  
In this paper we introduce an approach that lets us handle high-dimensional SPD matrices by constructing a lower-dimensional, more discriminative SPD manifold. To this end, we model the mapping from the high-dimensional SPD manifold to the low-dimensional one
with an orthonormal projection. In particular, we search for a projection that yields a low-dimensional manifold with maximum discriminative power encoded via an affinity-weighted similarity measure based on metrics on the manifold. Learning can then be expressed as an optimization problem on a Grassmann manifold. Our evaluation on several classification tasks shows that our approach 
leads to a significant accuracy gain over state-of-the-art methods.
\end{abstract}

\section{Introduction}
\label{sec:introduction}

This paper introduces an approach to embedding the Riemannian structure of Symmetric Positive Definite (SPD) matrices into a lower-dimensional, more discriminative Riemannian manifold.
SPD matrices  are becoming increasingly pervasive in various domains. For instance, diffusion tensors naturally arise in medical imaging~\cite{Pennec_IJCV_2006}. In computer vision, SPD matrices have been shown to provide powerful representations for images and videos via region covariances~\cite{Tuzel_ECCV_2006}. Such representations have been successfully employed to categorize textures~\cite{Tuzel_ECCV_2006,Harandi_ECCV_2012}, pedestrians~\cite{Tuzel_PAMI_2008}, faces~\cite{Pang_TCSVT_2008,Harandi_ECCV_2012}, actions and gestures~\cite{Sanin_WACV_2013}.

\def \CONC_SCALE {0.7}
\begin{figure}[!tb]	
	\centering	
	\includegraphics[width = \CONC_SCALE \textwidth]{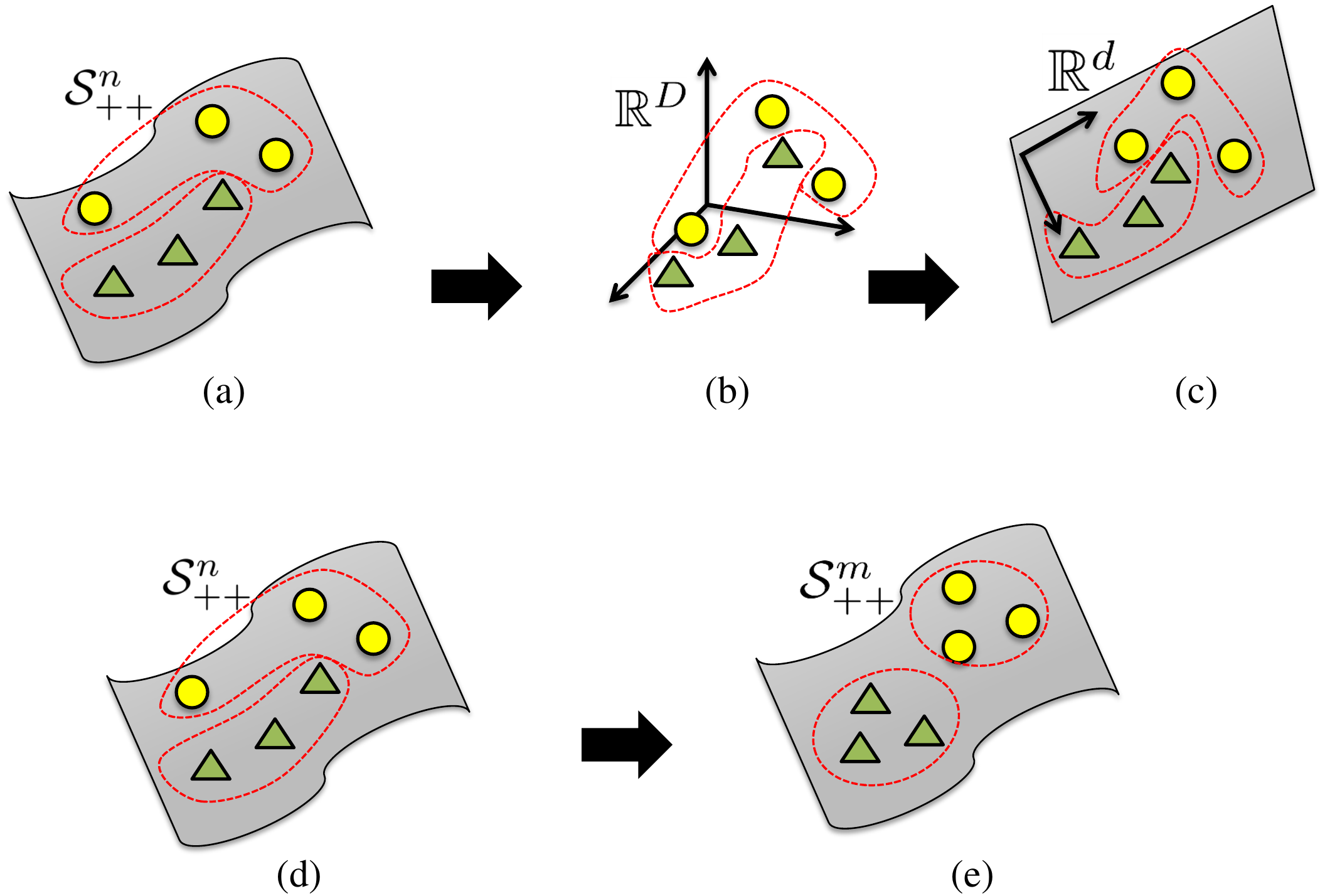}
	\caption{Conceptual comparison of typical dimensionality reduction methods on the manifold~\cite{Fletcher_2004_PGA,Wang_CVPR_2012_CDL} 
	and our approach.
	\textbf{Top row (existing techniques):} The original manifold (a) is first flattened either via tangent space computation or by Hilbert space embedding. 
	The flattened manifold (b) is then mapped to a lower-dimensional, optionally more discriminative space (c). The distortion incurred by the initial flattening may typically make this mapping more complicated.
	\textbf{Bottom row (our approach):} The original manifold (d) is directly transformed to a lower-dimensional, more discriminative manifold (e).
	}
	\label{fig:conceptual_diagram}
\end{figure}

SPD matrices can be thought of as an extension of positive numbers and form the interior of the positive semidefinite cone. It is possible to directly employ the Frobenius norm as a similarity measure between SPD matrices, hence analyzing problems involving such matrices via Euclidean geometry. However, as several studies have shown, undesirable phenomena may occur when Euclidean geometry is utilized to manipulate SPD matrices~\cite{Pennec_IJCV_2006,Tuzel_PAMI_2008,Sadeep_CVPR_2013}.
One example of this is the \textit{swelling effect} that occurs in diffusion tensor imaging (DTI), where a matrix represents the covariance of the local Brownian motion of water molecules~\cite{Pennec_IJCV_2006}: When considering Euclidean geometry to interpolate between two diffusion tensors, the determinant of the intermediate matrices may become strictly larger than the determinants of both original matrices, which is a physically unacceptable behavior. In~\cite{Pennec_IJCV_2006}, a Riemannian structure for SPD matrices was introduced to overcome the drawbacks of the Euclidean representation. This Riemannian structure is induced by the Affine Invariant Riemmanian Metric (AIRM), and is referred to as the SPD or tensor manifold. 

As shown in several studies~\cite{Pennec_IJCV_2006,Tuzel_PAMI_2008,Harandi_ECCV_2012,Sadeep_CVPR_2013}, accounting for the geometry of SPD manifolds can have a highly beneficial impact. However, it also leads to challenges in developing effective and efficient inference methods. The main trends in analyzing SPD manifolds are to either locally flatten them via tangent space approximations~\cite{Tuzel_PAMI_2008,Sanin_WACV_2013}, or embed them in higher-dimensional Euclidean spaces~\cite{Harandi_ECCV_2012,Caseiro_ECCV_2012,Sadeep_CVPR_2013}. In both cases, the computational cost of the resulting methods increases dramatically with the dimension of the SPD matrices. As a consequence, very low-dimensional SPD matrices are typically employed (\eg, region covariance descriptors obtained from a few low-dimensional features), with the exception of a few studies where medium-size matrices were used~\cite{Pang_TCSVT_2008,Harandi_ECCV_2012}. While the matrices obtained from low-dimensional features have proven sufficient for specific problems, they are bound to be less powerful and discriminative than the high-dimensional features typically used in computer vision.

To overcome this limitation, here, we introduce an approach that lets us handle high-dimensional SPD matrices. In particular, from a high-dimensional SPD manifold, we construct a lower-dimensional, more discriminative SPD manifold. While some manifold-based dimensionality reduction techniques have been proposed~\cite{Fletcher_2004_PGA,Wang_CVPR_2012_CDL}, as illustrated in Fig.~\ref{fig:conceptual_diagram}, they typically yield a Euclidean representation of the data and rely on flattening the manifold, which incurs distortions. In contrast, our approach directly works on the original manifold and exploits its geometry to learn a representation that {\bf (i)} still benefits from useful properties of SPD manifolds, and {\bf (ii)} can be used in conjunction with existing manifold-based recognition techniques to make them more practical and effective.

More specifically, given training SPD matrices, we search for a projection from their high-dimensional SPD manifold to a low-dimensional one such that the resulting representation maximizes an affinity-weighted similarity between pairs of matrices. In particular, we exploit the class labels to define an affinity measure, and employ either the AIRM, or the Stein divergence~\cite{Sra_NIPS_2012} to encode the similarity between two SPD matrices. Due to the affine invariance property of the AIRM and of the Stein divergence, any full rank projection would yield an equivalent representation. This allows us, without loss of generality, to model the projection with an orthonormal matrix, and thus express learning as an unconstrained optimization problem on a Grassmann manifold, which can be effectively optimized using a conjugate gradient method on the manifold. 

We demonstrate the benefits of our approach on several tasks where the data can be represented with high-dimensional SPD matrices. In particular, our method outperforms state-of-the-art techniques on three classification tasks: image-based material categorization and face recognition, and action recognition from 3D motion capture sequences. A Matlab implementation of our algorithm is available from the first author's webpage.

\section{Related Work}

We now discuss in more details the three techniques that also tackle dimensionality reduction of manifold-valued data.

Principal Geodesic Analysis (PGA) was introduced in~\cite{Fletcher_2004_PGA}  as a generalization of Principal Component Analysis (PCA) to Riemannian manifolds. PGA identifies the tangent space whose corresponding subspace maximizes the variability of the data on the manifold. PGA, however, is equivalent to flattening the Riemannian manifold by taking its tangent space at the Karcher, or Fr\'{e}chet, mean of the data. As such, it does not fully exploit the structure of the manifold. Furthermore, PGA, as PCA, cannot exploit the availability of class labels, and may therefore be sub-optimal for classification.

In~\cite{Wang_CVPR_2012_CDL}, the Covariance Discriminative Learning (CDL) algorithm was proposed to embed the SPD manifold into a Euclidean space. In contrast to PGA, CDL utilizes class labels to learn a discriminative subspace using Partial Least Squares (PLS) or Linear Discriminant Analysis (LDA). However, CDL relies on mapping the SPD manifold to the space of symmetric matrices via the principal matrix logarithm.
While this embedding has some nice properties (\eg, diffeomorphism), it can also be thought of as embedding 
the SPD manifold into its tangent space at the identity matrix. Therefore, although supervised, CDL also exploits data potentially distorted by the use of a single tangent space, as PGA.

Finally, in~\cite{Goh_CVPR_2008}, several Nonlinear Dimensionality Reduction techniques were extended to their Riemannian counterparts. This was achieved by introducing various Riemannian geometry concepts, such as Karcher mean, tangent spaces and geodesics, in Locally Linear Embedding (LLE), Hessian LLE and Laplacian Eigenmaps. The resulting algorithms were applied to several unsupervised clustering tasks. Although these methods can, in principle, be employed for supervised classification, they are limited to the transductive setting since they do not define any parametric mapping to the low-dimensional space.

In this paper, we learn a mapping from a high-dimensional SPD manifold to a lower-dimensional one without relying on tangent space approximations of the manifold. Our approach therefore accounts for the structure of the manifold and can simultaneously exploit class label information. The resulting mapping lets us effectively handle high-dimensional SPD matrices for classification purposes. Furthermore, by mapping to another SPD manifold, our approach can serve as a pre-processing step to other Riemannian-based approaches, such as the manifold sparse coding of~\cite{Harandi_ECCV_2012}, thus making them practical to work with more realistic, high-dimensional features. Note that, while our formulation is inspired from graph embedding methods in Euclidean spaces, \eg,~\cite{Yan_PAMI_2007}, here we work with data lying on more challenging non-linear manifolds.

To the best of our knowledge, this  is the first work that shows how a high-dimensional SPD manifold can be 
transformed into another SPD manifold with lower intrinsic dimension. Note that a related idea, but with a very different approach, was introduced in~\cite{jung2012analysis} to decompose high-dimensional spheres 
into submanifolds of decreasing dimensionality.

\section{Riemannian Geometry of SPD Manifolds}
\label{sec:preliminaries}

In this section, we discuss some notions of geometry of SPD manifolds. Throughout this paper we will use the following notation:
$\SPD{n}$ is the space of real $n \times n$ SPD matrices; $\mathbf{I}_n \in \mathbb{R}^{n \times n}$ is the identity matrix; 
$\mathrm{GL}(n)$ is the general linear group, \ie, the group of real invertible $n \times n$ matrices.

\begin{definition} \label{def:SPD}
A real and symmetric matrix $\Mat{X} \in \mathbb{R}^{n \times n}$ is said to be SPD if $\Vec{v}^T\Mat{X}\Vec{v}$ 
is positive for any non-zero $\Vec{v} \in \mathbb{R}^n$. 
\end{definition}

The space of $n \times n$ SPD matrices
is obviously not a vector space since multiplying an SPD matrix by a negative scalar results in a matrix which does not belong to $\SPD{n}$.
Instead, $\SPD{n}$ forms the interior of a convex cone in the $n^2$-dimensional Euclidean space.
The $\SPD{n}$ space is mostly studied when endowed with a Riemannian metric and thus forms a Riemannian manifold~\cite{Pennec_IJCV_2006}. 
A natural way to measure closeness on a manifold is by considering the geodesic distance between two points on the manifold.
Such a distance is defined as the length of the shortest curve connecting the two points.
The shortest curves are known as geodesics and are analogous to straight lines in $\mathbb{R}^n$.
The Affine Invariant Riemannian Metric (AIRM) is probably the most popular Riemannian structure for analyzing SPD 
matrices~\cite{Pennec_IJCV_2006}. Let $\Mat{P}$ be a point on $\SPD{n}$. The AIRM for two tangent vectors 
$\Vec{v},\Vec{w} \in T_{\Mat{P}}{\SPD{n}}$ is defined as
\begin{equation}
	\langle \Vec{v}, \Vec{w} \rangle_\Mat{P} \coloneqq  
	\langle \Mat{P}^{-1/2}\Vec{v}\Mat{P}^{-1/2}, \Mat{P}^{-1/2}\Vec{w}\Mat{P}^{-1/2} \rangle
	= \tr \left( \Mat{P}^{-1} \Vec{v} \Mat{P}^{-1} \Vec{w}\right)\;.
	\label{eqn:AIRM_equ}
\end{equation}

\begin{definition} \label{def:AIRM_geodesic}
	The geodesic distance $\delta_g: \mathcal{S}_{++}^{n} \times \mathcal{S}_{++}^{n} \rightarrow [0,\infty)$ induced by the AIRM
	is defined as
		\begin{align}
    	\delta_g^2(\Mat{X},\Mat{Y}) &=   \|\log(\Mat{X}\Mat{Y}^{-1})\|_F^2\;, 
    	\label{eqn:AIRM_geo}
		\end{align}%
where $\log(\cdot)$ is the matrix principal logarithm. 
\end{definition}
More recently, Sra introduced the Stein metric on SPD manifolds~\cite{Sra_NIPS_2012}:

\begin{definition} \label{def:stein_divergence}
	The Stein metric $\delta_S: \mathcal{S}_{++}^{n} \times \mathcal{S}_{++}^{n} \rightarrow [0,\infty)$ is a 
	symmetric type of Bregman divergence and is defined as
		\begin{align}
    	\delta_S^2(\Mat{X},\Mat{Y}) &= \ln \det \bigg( \frac{\Mat{X}+\Mat{Y}}{2} \bigg) 
    	- \frac{1}{2}  \ln \det( \Mat{X}\Mat{Y} )  \;. 
    	\label{eqn:Stein_Div}
		\end{align}%
\end{definition}

The Stein metric shows several similarities to the geodesic induced by the AIRM while being less expensive to compute~\cite{Cherian:PAMI:2013}.
In addition to the properties studied by Sra~\cite{Sra_NIPS_2012}, we provide the following important theorem which relates the
length of curves under the two metrics.
\begin{theorem} \label{thm:curve_equiv_stein_thm}
The length of any given curve is the same under $\delta_g$ and	$\delta_s$ up to a scale of $2\sqrt{2}$.
\end{theorem}
\begin{proof}
	Given in appendix~\ref{sec:proof_int_metric}. \qed
\end{proof}

One of the motivations for projecting a higher-dimensional SPD manifold to a 	
lower-dimensional one is to preserve the properties of  $\delta_g^2$ and $\delta_S^2$~\cite{Pennec_IJCV_2006,Sra_NIPS_2012}. 
One important such property, especially in computer vision, is affine invariance~\cite{Pennec_IJCV_2006}.

\paragraph{{\bf Property 1}}(Affine invariance). For any $\Mat{M} \in \mathrm{GL}(n)$,
	\begin{align*}			
		\delta_g^2(\Mat{X},\Mat{Y}) &= \delta_g^2(\Mat{M}\Mat{X}\Mat{M}^T,\Mat{M}\Mat{Y}\Mat{M}^T), \\
		\delta_S^2(\Mat{X},\Mat{Y}) &= \delta_S^2(\Mat{M}\Mat{X}\Mat{M}^T,\Mat{M}\Mat{Y}\Mat{M}^T).
	\end{align*}

This property postulates that the metric between two SPD matrices is unaffected by the action of the affine group. In the specific case where the SPD matrices are region covariance descriptors~\cite{Tuzel_ECCV_2006}, this implies that the distance between two descriptors will remain unchanged after an affine transformation of the image features, such as a change of illumination when using RGB values. Note that, in addition to this specific implication, we will also exploit the affine invariance property for a different purpose when deriving our learning algorithm in the next section.

\section{Geometry-Aware Dimensionality Reduction}
\label{sec:tensor_learning}
	
We now describe our approach to learning an embedding of high-dimensional SPD matrices to a more discriminative, low-dimensional SPD manifold. More specifically, given a matrix $\Mat{X} \in \SPD{n}$, we seek to learn the parameters $\Mat{W} \in \mathbb{R}^{n \times m}$, $m<n$, of a generic mapping $f: \SPD{n} \times \mathbb{R}^{n \times m} \rightarrow \SPD{m}$, which we define as 
\begin{equation}
	f(\Mat{X},\Mat{W}) = \Mat{W}^T\Mat{X}\Mat{W}.
	\label{eqn:generic_mapping}
\end{equation}
Clearly, if $\SPD{n} \ni \Mat{X} \succ 0$ and  $\Mat{W}$ has full rank, $\SPD{m} \ni \Mat{W}^T\Mat{X}\Mat{W} \succ 0$.
	
Given a set of SPD matrices $\mathcal{X} = \left\{\Mat{X}_1,\cdots, \Mat{X}_p \right\}$, where each matrix $\Mat{X}_i \in \SPD{n}$, our goal is to find a transformation $\Mat{W}$ such that the resulting low-dimensional SPD manifold preserves some interesting structure of the original data. Here, we encode this structure via an undirected graph defined by a real symmetric affinity matrix $\Mat{A} \in \mathbb{R}^{p \times p}$. The element $\Mat{A}_{ij}$ of this matrix measures some notion of affinity between matrices $\Mat{X}_i$ and $\Mat{X}_j$, and may be negative. We will discuss the affinity matrix in more details in Section~\ref{sec:affinity}.

Given $\Mat{A}$, we search for an embedding such that the affinity between pairs of SPD matrices is reflected by a measure of similarity on the low-dimensional SPD manifold. In this paper, we propose to make use of either the AIRM or the Stein metric to encode (dis)similarities between SPD matrices. For each pair $(i,j)$ of training samples, this lets us write a cost function of the form
\begin{equation}
	\mathcal{J}_{ij}(\Mat{W}; \Mat{X}_i,\Mat{X}_j) = \Mat{A}_{ij} \delta^2 \left(\Mat{W}^T\Mat{X}_i\Mat{W},\Mat{W}^T\Mat{X}_j\Mat{W}\right)\;,
	\label{eqn:dtl_local_cost}
\end{equation}
where $\delta$ is either $\delta_g$ or $\delta_S$.	These pairwise costs can then be grouped together in a global empirical cost function
\begin{equation}
	\label{eqn:dtl_cost_fun}
	L(\Mat{W}) = \sum_{i,j} \mathcal{J}_{ij}(\Mat{W};\Mat{X}_i,\Mat{X}_j),
\end{equation}
which we seek to minimize w.r.t. $\Mat{W}$.
	
To avoid degeneracies and ensure that the resulting embedding forms a valid SPD manifold, \ie, $\Mat{W}^T\Mat{X}\Mat{W} \succ 0,~ \forall \Mat{X} \in \SPD{n}$, we need $\Mat{W}$ to have full rank. Here, we enforce this requirement by imposing orthonormality constraints on 
$\Mat{W}$, \ie, $\Mat{W}^T\Mat{W} = \mathbf{I}_m$. Note that, with either the AIRM or the Stein divergence, this entails no loss of generality. 
Indeed, any full rank matrix $\tilde{\Mat{W}}$ can be expressed as $\Mat{M}\Mat{W}$, with $\Mat{W}$ an orthonormal matrix and $\Mat{M}\in \mathrm{GL}(n)$. The affine invariance property of the AIRM and of the Stein metric therefore guarantees that
\begin{equation*}
	\mathcal{J}_{ij}(\tilde{\Mat{W}}; \Mat{X}_i,\Mat{X}_j) = \mathcal{J}_{ij}(\Mat{M}\Mat{W}; \Mat{X}_i,\Mat{X}_j)  = \mathcal{J}_{ij}(\Mat{W}; \Mat{X}_i,\Mat{X}_j)\;.
\end{equation*}	

Finally, learning can be expressed as the minimization problem
\begin{align}
	\label{eqn:tl_main_equ}
	\Mat{W}^{\ast} = ~&\underset{\Mat{W} \in  \mathbb{R}^{n \times m}}{\arg\min}
	\sum_{i,j} \Mat{A}_{ij} \delta^2\left( \Mat{W}^T\Mat{X}_i\Mat{W}, \Mat{W}^T\Mat{X}_j\Mat{W} \right) ~~~
	{\rm s.t.}~\Mat{W}^T\Mat{W} = \mathbf{I}_m\;.
\end{align}
In the next section, we describe an effective way of solving~\eqref{eqn:tl_main_equ} via optimization on a (different) Riemannian manifold.
	
\subsection{Optimization on Grassmann Manifolds}

Recent advances in optimization methods formulate problems with orthogonality constraints as optimization problems on
Stiefel or Grassmann manifolds~\cite{Absil_2008}. More specifically, the geometrically correct setting for the  
minimization problem $\min L(\Mat{W})$ with the orthogonality constraint $\Mat{W}^T\Mat{W} = \mathbf{I}_m$
is, in general, on a Stiefel manifold. However, if the cost function $L(\Mat{W})$ possesses the property that for
any rotation matrix $\Mat{R}$ (\ie, $\Mat{R} \in \mathrm{SO}(m), \Mat{R}\Mat{R}^T = \Mat{R}^T\Mat{R} = \mathbf{I}_m$), 
$L(\Mat{W}) = L(\Mat{W}\Mat{R})$, then the problem is on a Grassmann manifold.
	
Since both the AIRM and the Stein metric are affine invariant, we have
\begin{equation*}
	\mathcal{J}(\Mat{X}_i,\Mat{X}_j,\Mat{W}) = \mathcal{J}(\Mat{X}_i,\Mat{X}_j,\Mat{W}\Mat{R}),
\end{equation*}	
\noindent
and thus $L(\Mat{W}) = L(\Mat{W}\Mat{R})$, which therefore identifies~\eqref{eqn:tl_main_equ} as an (unconstrained) optimization problem on the Grassmann manifold $\mathcal{G}(m,n)$.	
	
In particular, here, we utilize a nonlinear Conjugate Gradient (CG) method on Grassmann manifolds to minimize~\eqref{eqn:tl_main_equ}. A  
brief description of the steps of this algorithm is provided in appendix~\ref{sec:conjugate_gradient}. For a more detailed treatment, we refer the reader to~\cite{Absil_2008}. 
As for now, we just confine ourselves to saying that nonlinear CG on Grassmann manifolds requires the $n \times m$
Jacobian matrix of $L(\Mat{W})$ w.r.t. $\Mat{W}$. For the Stein metric, this Jacobian matrix can be obtained by noting that 
\begin{equation}
	D_\Mat{W}  \ln\det\big(\Mat{W}^T\Mat{X}\Mat{W}\big)  = 2\Mat{X}\Mat{W}\big( \Mat{W}^T\Mat{X}\Mat{W}\big)^{-1}\;,
	\label{eqn:logdet_gradient}
\end{equation}	
which lets us identify the Jacobian of the Stein divergence as
\begin{align*}
	D_\Mat{W} \delta_S^2 \big( \Mat{W}^T\Mat{X}_i\Mat{W}, \Mat{W}^T\Mat{X}_j\Mat{W} \big) &= 
	(\Mat{X}_i+\Mat{X}_j)\Mat{W}(\Mat{W}^T\frac{\Mat{X}_i+\Mat{X}_j}{2}\Mat{W})^{-1} \notag \\
	&-\Mat{X}_i\Mat{W}(\Mat{W}^T\Mat{X}_i\Mat{W})^{-1}
	-\Mat{X}_j\Mat{W}(\Mat{W}^T\Mat{X}_j\Mat{W})^{-1}\;.
\end{align*}
For the AIRM, we can exploit the fact that $\tr\left( \log(\Mat{X}) \right) = \ln \det\big ( \Mat{X} \big), \forall\Mat{X} \in \SPD{n}$. We can then derive the Jacobian by utilizing Eq.~\ref{eqn:logdet_gradient}, which yields

\begin{align*}
	&D_\Mat{W}\Big(\delta^2_{g} \left( \Mat{W}^T\Mat{X}_i\Mat{W}, \Mat{W}^T\Mat{X}_j\Mat{W} \right) \Big) =
	D_\Mat{W}\Big( \Big\| \log \Big(\Mat{W}^T\Mat{X}_i\Mat{W}\big(\Mat{W}^T\Mat{X}_j\Mat{W} \big)^{-1} 
	 \Big) \Big\|_F^2\Big) \notag\\
	&= 2D_\Mat{W}\bigg\{ \tr \bigg( \log \Big(  
	\Mat{W}^T\Mat{X}_i\Mat{W} \big(\Mat{W}^T\Mat{X}_j\Mat{W} \big)^{-1} \Big) \bigg) \bigg\} 
	\cdot\log \Big( 
	\Mat{W}^T\Mat{X}_i\Mat{W} \big(\Mat{W}^T\Mat{X}_j\Mat{W} \big)^{-1} \Big) 		\notag\\
	&= 2D_\Mat{W}\bigg( \ln \det\Big(  \Mat{W}^T\Mat{X}_i\Mat{W} \big(\Mat{W}^T\Mat{X}_j\Mat{W} \big)^{-1} \Big) \bigg)
	\log \Big( 
	\Mat{W}^T\Mat{X}_i\Mat{W} \big(\Mat{W}^T\Mat{X}_j\Mat{W} \big)^{-1} \Big) 		\notag\\							
	&=	4\Big(\Mat{X}_i\Mat{W}(\Mat{W}^T\Mat{X}_i\Mat{W})^{-1} -
			\Mat{X}_j\Mat{W}(\Mat{W}^T\Mat{X}_j\Mat{W})^{-1}\Big)
			\log\Big(\Mat{W}^T\Mat{X}_i\Mat{W}\big(\Mat{W}^T\Mat{X}_j\Mat{W}\big)^{-1}\Big)\,.	
\end{align*}
The pseudo-code for our SPD manifold learning (SPD-ML) method is given in Algorithm~\ref{alg:tl_alg}, where $\nabla_WL(\Mat{W})$ denotes the gradient on the manifold obtained from the Jacobian $D_WL(\Mat{W})$, and $\tau(\Mat{H},\Mat{W}_0,\Mat{W}_1)$ denotes the parallel transport of tangent vector $\Mat{H}$ from $\Mat{W}_0$ to $\Mat{W}_1$ (see appendix~\ref{sec:conjugate_gradient} for details).

\begin{algorithm}[!tb]		
	\caption{\small SPD Manifold Learning (SPD-ML).}
	\label{alg:tl_alg}
	\begin{algorithmic}
	\vspace{0.1cm}
	\STATE {\bfseries Input:}\\
	{A set of SPD matrices $\{ \Mat{X}_i \}_{i=1}^p,~\Mat{X}_i~\in~\SPD{n}$\\ 
	The corresponding labels $\{ y_i \}_{i=1}^p,~y_i \in \{1,2,\cdots,C\}$\\
    The dimensionality $m$ of the induced manifold
	}
	\vspace{0.1cm}
	\STATE {\bfseries Output:}\\
	{The mapping $\Mat{W} \in \mathcal{G}(m,n)$
	}
	\vspace{0.25cm}
	\STATE Generate $\Mat{A}$ using \eqref{eqn:Gw}, \eqref{eqn:Gb} and \eqref{eqn:affinity_Gw_Gb_combined}
	\STATE $\Mat{W}_{old} \gets \Mat{I}_{n\times m}$ (\ie, the truncated identity matrix)
	\STATE $\Mat{W} \gets \Mat{W}_{old}$
    \STATE $\Mat{H}_{old} \gets \Mat{0}$
	\REPEAT		
	    \STATE $\Mat{H} \gets -\nabla_{\Mat{W}} L(\Mat{W}) +\eta \tau(\Mat{H}_{old},\Mat{W}_{old},\Mat{W}) $
		\STATE Line search along the geodesic $\gamma(t)$ from $\Mat{W} = \gamma(0)$ in the direction $\Mat{H}$
		to find \mbox{$\Mat{W}^\ast = \underset{\Mat{W}}{\rm{argmin}}~L(\Mat{W})$}			
		\STATE $\Mat{H}_{old} \gets \Mat{H}$
		\STATE $\Mat{W}_{old} \gets \Mat{W}$	
		\STATE $\Mat{W} \gets \Mat{W}^\ast$
	\UNTIL{convergence}
	\end{algorithmic}
\end{algorithm}

\subsection{Designing the Affinity Matrix}
\label{sec:affinity}

Different criteria can be employed to build the affinity matrix $\Mat{A}$. 
In this work, we focus on classification problems on $\SPD{n}$ and therefore exploit class labels to construct $\Mat{A}$. Note, however, that our framework is general and also applies to  
unsupervised or semi-supervised settings. For example, in an unsupervised scenario, $\Mat{A}$ could be built from pairwise 
similarities (distances) on $\SPD{n}$. Solving~\eqref{eqn:tl_main_equ} could then be understood as finding a 	
mapping  where nearby data pairs on the original manifold $\SPD{n}$ remain close in the induced manifold $\SPD{m}$. 

Let us assume that each point $\Mat{X}_i \in \SPD{n}$ belongs to one of $C$ possible classes 
and denote its class label by $y_i$. Our aim is to define an affinity matrix that encodes the notions of 
intra-class and inter-class distances, and thus, when solving~\eqref{eqn:tl_main_equ}, yields a mapping 
that minimizes the intra-class distances while simultaneously maximizing the inter-class distances (\ie, a discriminative mapping).

More specifically, let $\left\{  (\Mat{X}_i, y_i) \right\}_{i=1}^{p}$ be the set of $p$ labeled training points, where $\Mat{X}_i \in \SPD{n}$ and $y_i \in \left\{ 1,2, \cdots, C \right\}$. The affinity of the training data on 
$\SPD{n}$ can be modeled by building a within-class similarity graph $\Mat{G}_w$
and a between-class similarity graph $\Mat{G}_b$.
In particular, we define $\Mat{G}_w$ and $\Mat{G}_b$
as binary matrices constructed from nearest neighbor graphs. This yields
\begin{equation}
   	\Mat{G}_w(i,j) =
   	\left\{
 		\begin{matrix}
       	1, & \mbox{if} \; \Mat{X}_i \in N_w(\Mat{X}_j ) \; \mbox{~or~} \; \Mat{X}_j \in N_w(\Mat{X}_i ) \\
   		0, & \mbox{otherwise}
   		\end{matrix}
   	\right.
\label{eqn:Gw}
\end{equation}%
\noindent
\begin{equation}
   	\Mat{G}_b(i,j) =
   	\left\{
   	\begin{matrix}
       	1, & \mbox{if} \; \Mat{X}_i \in N_b(\Mat{X}_j ) \; \mbox{~or~} \; \Mat{X}_j \in N_b(\Mat{X}_i ) \\
       	0, & \mbox{otherwise}
   	\end{matrix}
   	\right.
   	\label{eqn:Gb}
\end{equation}%
\noindent
where $N_w(\Mat{X}_i)$ is the set of $\nu_w$ nearest neighbors of $\Mat{X}_i$ that share the same label as $y_i$, 
and $N_b(\Mat{X}_i)$ contains the $\nu_b$ nearest neighbors of $\Mat{X}_i$
having different labels. The affinity matrix $\Mat{A}$ is then defined as
\begin{equation}
	\Mat{A} = \Mat{G}_w - \Mat{G}_b \;,
	\label{eqn:affinity_Gw_Gb_combined}
\end{equation}
which resembles the Maximum Margin Criterion (MMC) of~\cite{LI_TNN_2006}. In practice, we set $\nu_w$ to the minimum number of points in each class and, to balance the influence of $\Mat{G}_w$ and $\Mat{G}_b$, choose $\nu_b \leq \nu_w$, with the specific value found by 
cross-validation. We analyze the influence of $\nu_b$ in appendix~\ref{sec:additional_exp}.

\subsection{Discussion in Relation to Region Covariance Descriptors}
\label{sec:further_dis}	

In our experiments, we exploited Region Covariance Matrices (RCMs)~\cite{Tuzel_ECCV_2006} as image descriptors. Here, we discuss some interesting properties of our algorithm when applied to these specific SPD matrices.

There are several reasons why RCMs are attractive to represent images and videos. First, RCMs provide a natural way to fuse various feature types. Second, they help reducing the impact of noisy samples in a region via their inherent averaging operation. Third, RCMs are independent of the size of the region, and can therefore easily be utilized to compare regions of different sizes. Finally, RCMs can be efficiently computed using integral images~\cite{Tuzel_PAMI_2008,Sanin_WACV_2013}.

Let $I$ be a $W \times H$ image, and $\mathbb{O} = \{\Vec{o}_i\}_{i=1}^{r}, \; \Vec{o}_i \in \mathbb{R}^n$ be a set of $r$ observations extracted from $I$, \eg, $\Vec{o}_i$ concatenates intensity values, gradients along the horizontal and vertical directions, filter responses,... for image pixel $i$. Let $\mu = \frac{1}{r} \sum_{i = 1}^{r} \Vec{o}_i$ be the mean value of the observations.
Then image $I$ can be represented by the $n \times n$ RCM 
\begin{align}
	\Mat{C}_{I} = \frac{1}{r-1} \sum_{i = 1}^{r} \left(\Vec{o}_i - \mu \right)\left(\Vec{o}_i - \mu \right)^T 
	= \Mat{O}\Mat{J}\Mat{J}^T\Mat{O}^T\;, 
	\label{eqn:cov_desc}
\end{align}
where $\Mat{J} = {r}^{-3/2}(r\mathbf{I}_{r} - \Mat{1}_{r \times r})$.
To have a valid RCM, $r \geq n$, otherwise $\Mat{C}_{I}$ would have zero eigenvalues, which would make both $\delta_g^2$ 
and $\delta_S^2$ indefinite. 

After learning the projection $\Mat{W}$, the low-dimensional representation of image $I$ is given by $\Mat{W}^T\Mat{O}\Mat{J}\Mat{J}^T\Mat{O}^T\Mat{W}$.
This reveals two interesting properties of our learning scheme. {\bf 1)} The resulting representation can also be thought of as an RCM with $\Mat{W}^T\Mat{O}$ 
as a set of low-dimensional observations. Hence, in our framework, we can create a valid $\SPD{m}$ manifold with only $m$ observations instead of at least $n$ in the original input space. This is not the case for other algorithms, which require having matrices on $\SPD{n}$ as input. In appendix~\ref{sec:additional_exp}, we study the influence of the number of observations on recognition accuracy. {\bf 2)} Applying $\Mat{W}$ directly the set of observations reduces the computation time of creating the final RCM on $\SPD{m}$. This is due to the fact that the computational complexity of computing an RCM is quadratic in the dimensionality of the features.

\section{Empirical Evaluation}
\label{sec:experiments}

In this section, we study the effectiveness of our SPD manifold learning approach. In particular, as mentioned earlier, we focus on classification and present results on two image datasets and one motion capture dataset. In all our experiments, the dimensionality of the low-dimensional SPD manifold was determined by cross-validation. Below, we first briefly describe the different classifiers used in these experiments, and then discuss our results.

\paragraph{Classification algorithms:}
The SPD-ML algorithm introduced in Section~\ref{sec:tensor_learning} allows us to obtain a low-dimensional, more discriminative SPD manifold from a high-dimensional one. Many different classifiers can then be used to categorize the data on this new manifold. In our experiments, we make use of two such classifiers. First, we employ a simple nearest neighbor classifier based on the manifold metric (either AIRM or Stein). This simple classifier clearly evidences the benefits of mapping the original Riemannian structure to a lower-dimensional one. Second, we make use of the Riemannian sparse coding algorithm of~\cite{Harandi_ECCV_2012} (RSR). This algorithm exploits the notion of sparse coding to represent a query SPD matrix using a codebook of SPD matrices. In all our experiments, we formed the codebook purely from the training data, \ie, no dictionary learning was employed. Note that RSR relies on a kernel derived from the Stein metric. We therefore only applied it to the Stein metric-based version of our algorithm. We refer to the different algorithms evaluated in our experiments as:

\begin{itemize}
	\renewcommand{\labelitemi}{\scriptsize$$}
	\item \textbf{NN-Stein:} Stein metric-based Nearest Neighbor classifier.
	\item \textbf{NN-AIRM:} AIRM-based Nearest Neighbor classifier.
	\item \textbf{NN-Stein-ML:} Stein metric-based Nearest Neighbor classifier on the low-dimensional SPD manifold obtained with our approach.
	\item \textbf{NN-AIRM-ML:} AIRM-based Nearest Neighbor classifier on the low-dimensional SPD manifold obtained with our approach.
	\item \textbf{RSR:} Riemannian Sparse Representation~\cite{Harandi_ECCV_2012}.
	\item \textbf{RSR-ML:} Riemannian Sparse Representation on the low-dimensional SPD manifold obtained with our approach.
\end{itemize}

In addition to these methods, we also provide the results of the PLS-based Covariance Discriminant Learning (CDL) technique 
of~\cite{Wang_CVPR_2012_CDL}, as well as of the state-of-the-art baselines of each specific dataset.

\subsection{Material Categorization}
\label{sec:exp_material_cat}

For the task of material categorization, we used the UIUC dataset~\cite{UIUC_Dataset}. 
The UIUC material dataset contains 18 subcategories of materials taken in the wild from 
four general categories (see Fig.~\ref{fig:UIUC_Dataset}): \textit{bark}, \textit{fabric}, \textit{construction materials}, and 
\textit{outer coat of animals}.
Each subcategory has 12 images taken at various scales. 
Following standard practice, half of the images from each subcategory was randomly chosen as training data, and the rest was 
used for testing. We report the average accuracy over 10 different random partitions.

Small RCMs, such as those used for texture recognition in~\cite{Harandi_ECCV_2012}, are hopeless here due to the complexity of the task.
Recently, SIFT features~\cite{SIFT_IJCV_2004} have been shown to be 
robust and discriminative for material classification~\cite{UIUC_Dataset}. Therefore, we constructed RCMs of size $155\times155$
using 128 dimensional SIFT features (from grayscale images) and 27 dimensional color descriptors. 
To this end, we resized all the images to $400 \times 400$ and computed dense SIFT descriptors on a regular grid with 4 pixels spacing. The color descriptors were obtained by simply stacking colors from $3\times3$ patches centered at the grid points. Each grid point therefore yields one 155-dimensional observation $\Vec{o}_i$ in Eq.~\ref{eqn:cov_desc}. The parameters for this experiments were set to $\nu_w = 6$ (minimum number of samples in a class), and $\nu_b = 3$ obtained by 5-fold cross-validation.

Table~\ref{tab:table_UIUC_performance} compares the performance of our different algorithms and of the state-of-the-art method on this dataset (SD)~\cite{UIUC_Dataset}. The results show that appropriate manifold-based methods (\ie, RSR and CDL) with the original $155 \times 155$ RCMs already outperform SD, while NN on the same manifold yields worse performance. However, after applying our learning algorithm, NN not only outperforms SD significantly, but also outperforms both CDL and RSR. RSR on the learned SPD manifold (RSR-ML) further boosts the accuracy to $66.6\%$.

To further evidence the importance of geometry-aware dimensionality reduction, we replaced our low-dimensional RCMs with RCMs obtained by applying PCA directly on the $155$ dimensional features. The AIRM-based NN classifier used on these RCMs gave $42.1\%$ accuracy (best performance over different PCA dimensions). While this is better than the performance in the original feature space (\ie, $35.6\%$), it is significantly lower than the accuracy of our NN-AIRM-ML approach (\ie, $58.3\%$). Finally, note that performing NN-AIRM on the original data required 490s on a 3GHz machine with Matlab. After our dimensionality reduction scheme, this only took 9.7s.

\begin{figure}[!t]
\centering		
	\includegraphics[width = 0.8\textwidth]{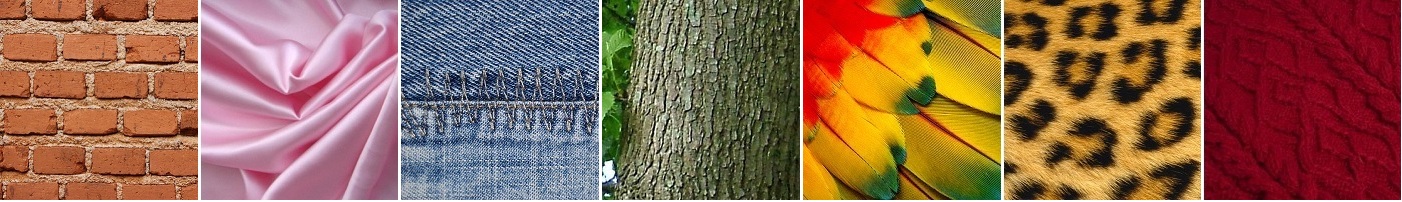}
	\caption{Samples from the UIUC material dataset.}
	\label{fig:UIUC_Dataset}
\end{figure}
\begin{figure}[t]
\begin{minipage}[t]{\columnwidth}
	\centering
	\begin{minipage}[t]{0.45 \linewidth}
		\begin{table}[H] 			
	  	\centering
    	\begin{tabular}{lc}
    		\toprule
	    	{\bf Method} &{\bf Accuracy }\\
    		\toprule  		
    		{\bf SD~\cite{UIUC_Dataset}}                          	 &$43.5\% \pm N/A$\\
	    	{\bf CDL~\cite{Wang_CVPR_2012_CDL}	}          		 	 &$52.3\% \pm 4.3$\\
    		\midrule
    		{\bf NN-Stein}                        	 				 &$35.8\% \pm 2.6$\\
	    	{\bf NN-Stein-ML}                          				 &$58.1\% \pm 2.8$\\
    		\midrule
    		{\bf NN-AIRM}                         	 				 &$35.6\% \pm 2.6$\\
	    	{\bf NN-AIRM-ML}                           				 &$58.3\% \pm 2.3$\\
    		\midrule
    		{\bf RSR~\cite{Harandi_ECCV_2012}	}          		 	 &$52.8\% \pm 2.1$\\    		
	    	{\bf RSR-ML}                         	 			 &$\bf{66.6\% \pm 3.1}$\\
    		\bottomrule	
		    \end{tabular}    	
		    \vspace{0.2cm}	   		
    		\caption   {Mean recognition accuracies with standard deviations for the UIUC material dataset~\cite{UIUC_Dataset}.}	
    		\label{tab:table_UIUC_performance} 
		\end{table}	
	\end{minipage}
	\hspace{4ex}
	\begin{minipage}[t]{0.45 \linewidth}
		\begin{table}[H]		   	    
		 	\centering
			\vspace{0.39cm}
		    \begin{tabular}{lc}
    			\toprule
		    	{\bf Method} &{\bf Accuracy }\\
    			\toprule  		
		    	{\bf CDL~\cite{Wang_CVPR_2012_CDL}}    					 &$79.8\%$\\
    			\midrule
		    	{\bf NN-Stein}                               			 &$61.7\%$\\
    			{\bf NN-Stein-ML}                            			 &$68.6\%$\\
		    	\midrule
    			{\bf NN-AIRM}                              				 &$62.8\%$\\
		    	{\bf NN-AIRM-ML}                           				 &$67.6\%$\\
    			\midrule
		    	{\bf RSR~\cite{Harandi_ECCV_2012}}                 		 &$76.1\%$\\    			
    			{\bf RSR-ML}                         	 			 &$\bf{81.9\%}$\\
				\bottomrule	
		    \end{tabular}
		    \vspace{0.2cm}	   		
		    \caption    { Recognition accuracies for the HDM05-MOCAP dataset~\cite{HDM05_Doc}.}	 	
		    \label{tab:table_HDM05_performance}		   
		 \end{table}
	\end{minipage}	
\end{minipage}
\end{figure}
\def \MOCAP_SCALE  {0.5}
\begin{figure}[!t]
	\centering
	\includegraphics[width = \MOCAP_SCALE \columnwidth]{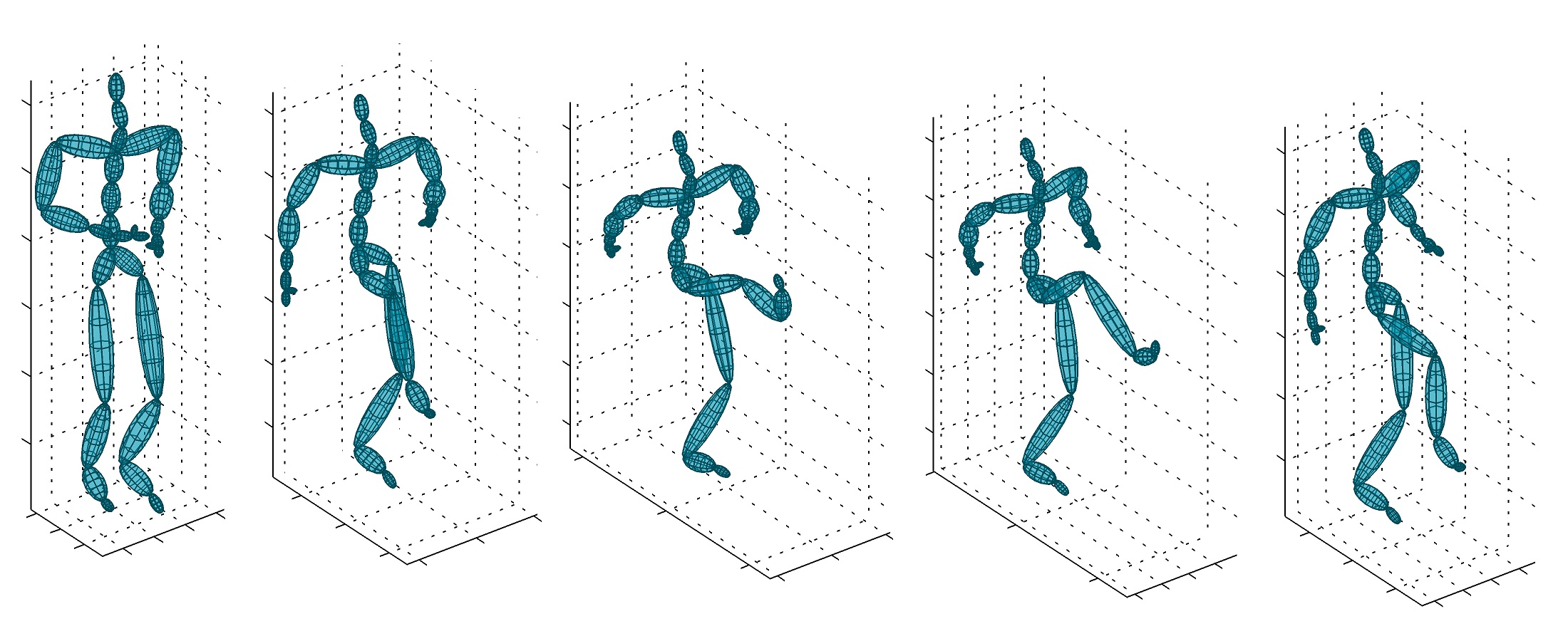}
	\caption{Kicking action from the HDM05 motion capture sequences database~\cite{HDM05_Doc}.}
	\label{fig:MoCap_Dataset}
\end{figure}
\subsection{Action Recognition from Motion Capture Data}
\label{sec:exp_action_rec}
As a second experiment, we tackled the problem of human action recognition from motion capture sequences using the HDM05 database~\cite{HDM05_Doc}. This database contains the following 14 actions: `clap above head', `deposit floor', `elbow to knee', `grab high', `hop both legs', `jog', `kick forward', `lie down floor', `rotate both arms backward', `sit down chair', `sneak', `squat', `stand up lie' and `throw basketball' (see Fig.~\ref{fig:MoCap_Dataset} for an example).
The dataset provides the 3D locations of 31 joints over time acquired at the speed of 120 frames per second.
We describe an action of a $K$ joints skeleton observed over $m$ frames by its joint covariance descriptor~\cite{Husse_IJCAI_2013}, which is an SPD matrix
of size $3K \times 3K$. This matrix is computed as in Eq.~\ref{eqn:cov_desc} by taking $\Vec{o}_i$ as the $93$-dimensional vector concatenating the 3D coordinates of the $31$ joints in frame $i$.

In our experiments, we used 2 subjects for training (\ie, 'bd' and 'mm') and the remaining 3 subjects for testing (\ie, 'bk', 'dg' and 'tr')\footnote{Note that this differs from the setup in~\cite{Husse_IJCAI_2013}, where 3 subjects were used for training and 2 for testing. However, with the setup of~\cite{Husse_IJCAI_2013} where an accuracy of $95.41\%$ was reported, all our algorithms resulted in about $99\%$ accuracy.}. This resulted in 118 training and 188 test sequences for this experiment. The parameters of our method were set to $\nu_w = 5$ (minimum number of samples in one class), and $\nu_b = 5$ by cross-validation.

We report the performance of the different methods on this dataset in Table~\ref{tab:table_HDM05_performance}. 
Again we can see that the accuracies of NN and RSR are significantly improved by our learning algorithm, and that our RSR-ML approach achieves the best accuracy of $81.9\%$.
As on the UIUC dataset, we also evaluated the performance RCMs built by reducing the dimensionality of the features
using PCA. This yielded an accuracy of $63.3\%$ with an AIRM-based NN classifier (best performance over different PCA dimensions). Again, while this slightly outperforms the accuracy of NN-AIRM (\ie, $62.8\%$), it remains clearly inferior to the performance of our NN-AIRM-ML algorithm (\ie, $67.6\%$).

\subsection{Face Recognition}
\label{sec:exp_face_rec}

For face recognition, we used the `b' subset of the FERET dataset~\cite{FERET_Dataset},
which contains 1800 images from 200 subjects. Following common practice~\cite{Harandi_ECCV_2012}, 
we used cropped images, downsampled to $64 \times 64$.
Fig.~\ref{fig:FERET_Dataset} depicts samples from the dataset.

We performed six experiments on this dataset. In all these experiments, the training data was composed of frontal faces with expression and illumination variations (\ie, images marked as `ba', `bj' and `bk'). The six experiments correspond to using six different non-frontal viewing angles as test data (\ie, images marked as `bc',`bd', `be', `bf', `bg' and `bh', respectively).

To represent a face image, we block diagonally concatenated three different $43 \times 43$ RCMs: one obtained from the entire image, one from the left half and one from the right half. This resulted in an RCM of size $129 \times 129$ for each image. Each $43 \times 43$ RCM was computed from the features
\noindent
\begin{equation*}
	\Vec{o}_{x,y} =	\left[~ I(x,y),~ x,~ y,~ |G_{0,0}(x,y)|,~\cdots,~|G_{4,7}(x,y)| ~\right]\;,
\end{equation*}%
\noindent
where {\small $I(x,y)$} is the intensity value at position $(x,y)$, $G_{u,v}{(x,y)}$ is the response of a 2D Gabor wavelet~\cite{Lee_PAMI_1996}
centered at $(x,y)$ with orientation $u$ and scale $v$, and $|\cdot|$ denotes the magnitude of a 
complex value. Here, following~\cite{Harandi_ECCV_2012}, we generated 40 Gabor filters at 8 orientations and 5 scales.

In addition to our algorithms, we evaluated the state-of-the-art Sparse Representation based Classification (SRC)~\cite{Wright_PAMI_2009} and
its Gabor-based extension (GSRC)~\cite{Yang_ECCV_2010_GSRC}. For SRC, we reduced the dimensionality of the data using PCA and chose the dimensionality that gave the best performance. For GSRC, we followed the recommendations of the authors to set the downsampling factor in the Gabor filters, but found that better results could be obtained with a larger $\lambda$ than the recommended one, and thus report these better results obtained with $\lambda = 0.1$. The parameters for our approach were set to $\nu_w = 3$ (minimum number of samples in one class), and $\nu_b = 1$ by cross-validation.

Table~\ref{tab:table_FERET_performance} reports the performance of the different methods. Note that both CDL and RSR outperform the Euclidean face recognition systems SRC and GSRC. Note also that even a simple Stein-based NN on $129 \times 129$ RCMs performs roughly on par with GSRC and better than SRC. More importantly, the representation learned with our SPD-ML algorithm yields significant accuracy gains when used with either NN or RSR for all different viewing angles, with more than $10\%$ improvement for some poses.

\def \FERET_SCALE {0.07}
\begin{figure}[!t]		
\begin{minipage}{\columnwidth}\centering
	\subfigure[ba]{
		\centering
		\includegraphics[width = \FERET_SCALE \linewidth]{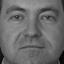}
	}
	\subfigure[bj]{
		\centering
		\includegraphics[width = \FERET_SCALE \linewidth]{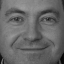}
	}
	\subfigure[bk]{
		\centering
		\includegraphics[width = \FERET_SCALE \linewidth]{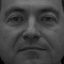}
	}
	\subfigure[bc]{
		\centering
		\includegraphics[width = \FERET_SCALE \linewidth]{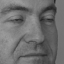}
	}
	\subfigure[bd]{
		\centering
		\includegraphics[width = \FERET_SCALE \linewidth]{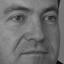}
	}
	\subfigure[be]{
		\centering
		\includegraphics[width = \FERET_SCALE \linewidth]{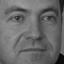}
	}				
	\subfigure[bf]{
		\centering
		\includegraphics[width = \FERET_SCALE \linewidth]{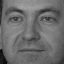}
	}
	\subfigure[bg]{
		\centering
		\includegraphics[width = \FERET_SCALE \linewidth]{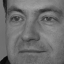}
	}
	\subfigure[bh]{
		\centering
		\includegraphics[width = \FERET_SCALE \linewidth]{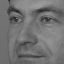}
	}												
\end{minipage}	
	\caption{ \small {Samples from the FERET dataset~\cite{FERET_Dataset}.}}
	\label{fig:FERET_Dataset}
\end{figure}
\begin{table*}[!t]
 	\centering
 	\begin{tabular}{lccccccc}
 		\toprule
 		{\bf Method} &{\bf bc }	&{\bf bd }	&{\bf be }	&{\bf bf }	&{\bf bg }	&{\bf bh }  &{\bf average acc.}\\
 		\toprule  		
 		{\bf SRC~\cite{Wright_PAMI_2009}}      &$9.5\%$	&$37.5\%$	&$77.0\%$	&$88.0\%$	&$48.5\%$	&$11.0\%$	&$45.3\% \pm 3.3$\\  
 		{\bf GSRC~\cite{Yang_ECCV_2010_GSRC}}  &$35.5\%$ &$77.0\%$	&$93.5\%$	&$97.0\%$	&$79.0\%$	&$38.0\%$ &$70.0\% \pm 2.7$\\
 		{\bf CDL~\cite{Wang_CVPR_2012_CDL}}    &$35.0\%$	&$87.5\%$	&$\bf 99.5\%$	&$\bf 100.0\%$	&$91.0\%$	&$34.5\%$ &$74.6\% \pm 3.1$\\
 		\midrule
 		{\bf NN-Stein}         				   &$29.0\%$ &$75.5\%$	&$94.5\%$	&$98.0\%$	&$83.5\%$	&$34.5\%$ &$69.2\% \pm 3.0$\\
 		{\bf NN-Stein-ML}         			   &$40.5\%$ &$88.5\%$	&$97.0\%$	&$99.0\%$	&$91.5\%$	&$44.5\%$ &$76.8\% \pm 2.7$\\
 		\midrule
 		{\bf NN-AIRM}         				   &$28.5\%$ &$72.5\%$	&$93.0\%$	&$97.5\%$	&$83.0\%$	&$35.0\%$ &$68.3\% \pm 3.0$\\
 		{\bf NN-AIRM-ML}         			   &$39.0\%$ &$84.0\%$	&$96.0\%$	&$99.0\%$	&$90.5\%$	&$45.5\%$ &$75.7\% \pm 2.6$\\
 		\midrule
 		{\bf RSR~\cite{Harandi_ECCV_2012}}     &$36.5\%$ &$79.5\%$	&$96.5\%$	&$97.5\%$	&$86.0\%$	&$41.5\%$ &$72.9\% \pm 2.7$\\    			
 		{\bf RSR-ML}					   &$\bf 49.0\%$ &$\bf 90.5\%$ &$98.5\%$ &$\bf 100\%$ &$\bf 93.5\%$ &$\bf 50.5\%$ &$\bf80.3\% \pm 2.4$\\
    	\bottomrule	
 	\end{tabular} 	
	\vspace{0.2cm}
 	\caption    {Recognition accuracies for the FERET face dataset~\cite{FERET_Dataset}.}
 	\label{tab:table_FERET_performance}
\end{table*}

\section{Conclusions and Future Work}
\label{sec:conclusion}
We have introduced a learning algorithm to map a high-dimensional SPD manifold into a lower-dimensional, more discriminative one. To this end, we have exploited a graph embedding formalism with an affinity matrix that encodes intra-class and inter-class distances, and where the similarity between two SPD matrices is defined via either the Stein divergence or the AIRM. Thanks to their invariance to affine transformations, these metrics have allowed us to model the mapping from the high-dimensional manifold to the low-dimensional one with an orthonormal projection. Learning could then be expressed as the solution to an optimization problem on a Grassmann manifold. Our experimental evaluation has demonstrated that the resulting low-dimensional SPD matrices lead to state-of-the art recognition accuracies on several challenging datasets.

In the future, we plan to extend our learning scheme to the unsupervised and semi-supervised scenarios. Finally, we believe that this work is a first step towards showing the importance of preserving the Riemannian structure of the data when performing dimensionality reduction, and thus going from one manifold to another manifold of the same type. We therefore intend to study how this framework can be applied to other types of Riemannian manifolds.

\appendix

\section{Proof of Length Equivalence}
\label{sec:proof_int_metric}

Here, we prove Theorem~1 from Section~3,
\ie, the equivalence between the length of any given curve under the geodesic distance $\delta_g$ and the Stein metric $\delta_S$ up to scale of $2\sqrt{2}$.
The proof of this theorem follows several steps.
We start with the definition of curve length and intrinsic metric.
Without any assumption on differentiability, let $(\mathcal{M},d)$ be a metric space.  
A curve in $\mathcal{M}$ is a continuous function $\gamma : [0, 1] \rightarrow \mathcal{M}$ and joins the starting point $\gamma(0) = x$
to the end point $\gamma(1) = y$. 
\begin{definition} \label{def:curve_length}
	The length of a curve $\gamma$ is the supremum of $l(\gamma ; \{t_i \})$ over all possible partitions \mbox{$\{t_i \}$},
	where \mbox{$0 = t_0 < t_1 < \cdots < t_{n-1} < t_n = 1$} and
	\mbox{{$l(\gamma ; \{t_i \}) = \sum_{i}d\left(\gamma(t_i),\gamma(t_{i-1})\right)$}}.
\end{definition}

\begin{definition} \label{def:intrinsic_metric_thm}
	The intrinsic metric $\widehat{\delta}(x,y)$ on $\mathcal{M}$
	is defined as the infimum of the lengths of all paths from $x$ to $y$.
\end{definition}
	
\begin{theorem}[~\cite{Hartley_IJCV_13}]
	If the intrinsic metrics induced by two metrics $d_1$ and $d_2$ are identical up to a scale $\xi$,
	then the length of any given curve is the same under both metrics up to $\xi$.
\end{theorem}
\begin{theorem}[~\cite{Hartley_IJCV_13}]
	If $d_1(x,y)$ and $d_2(x,y)$ are two metrics defined on a space $\mathcal{M}$ such that
	\begin{equation}
		\lim_{d_1(x,y) \rightarrow 0} \:\frac{d_2(x,y)}{d_1(x,y)} = 1.
		\label{eqn:intrinsic_metric0}
	\end{equation}
	uniformly (with respect to $x$ and $y$), then their intrinsic metrics are identical.
\end{theorem}
Therefore, here, we need to study the behavior of 
\begin{equation*}
	\lim_{\delta_S^2(\Mat{X},\Mat{Y}) \rightarrow 0} \:\frac{\delta_g^2(\Mat{X},\Mat{Y})}{\delta_S^2(\Mat{X},\Mat{Y})}\;
\end{equation*}
to prove our theorem on curve length equivalence.

\begin{proof}
	Let us first note that for an affine invariant metric $\delta$ on $\mathcal{S}_{++}^d$,
	\begin{equation*}
		\delta^2(\Mat{X},\Mat{Y}) = \delta^2(\Mat{I}_d,\Mat{D}^{-1/2}\Mat{L}^T\Mat{Y}\Mat{L}\Mat{D}^{-1/2}) 
		\triangleq \delta^2(\Mat{I}_d,\Mat{M})\;,
	\end{equation*} 
	where $\Mat{X} = \Mat{L}\Mat{D}\Mat{L}^T$ and $\Mat{L}\Mat{L}^T = \Mat{I}_d$. 
	Similarly, we can decompose $\Mat{M}$ as $\Mat{M} = \tilde{\Mat{L}}\tilde{\Mat{D}}\tilde{\Mat{L}}^T$, with 
	$\tilde{\Mat{L}}\tilde{\Mat{L}}^T = \tilde{\Mat{L}}^T\tilde{\Mat{L}} = \Mat{I}_d$, which yields
	\begin{equation*}
		\delta^2(\Mat{X},\Mat{Y}) = \delta^2(\Mat{I}_d,\tilde{\Mat{D}})\;.
	\end{equation*} 
Since all our matrices are positive definite, $\tilde{\Mat{D}}$ is a diagonal matrix with strictly positive values on its diagonal, 
and can be written as
\begin{equation*}
	\tilde{\Mat{D}} \triangleq \DIAG(\exp (t \Vec{\nu}))\;,
\end{equation*}
with $\Vec{\nu} \in \mathbb{R}^d$ and $t \in \mathbb{R}$. This definition can also be motivated by noting that the tangent vectors at $\mathbf{I}_d$ 
are symmetric matrices of the form $\tilde{\Mat{L}}\DIAG(t \Vec{\nu})\tilde{\Mat{L}}^T$. Applying the exponential map yields points on the manifold of the form $\tilde{\Mat{L}}\DIAG(\exp(t \Vec{\nu}))\tilde{\Mat{L}}^T$. 
As mentioned before, with an affine invariant metric, the dependency on $\tilde{\Mat{L}}$ and $\tilde{\Mat{L}}^T$ can be dropped.
		 
The previous discussion implies that we just need to study the behavior of the Stein metric around $\Mat{I}_d$ using a diagonal matrix
to draw any conclusion.	 We note that $\tilde{\Mat{D}} \rightarrow \mathbf{I}_d$ iff $t \rightarrow 0$. Therefore, given the definitions 
of $\delta_g$ and $\delta_S$ from Section~3 of the paper, we have		
\begin{align}
	&\lim_{\Mat{X} \rightarrow \Mat{Y}} \:\frac{\delta_g^2 (\Mat{X},\Mat{Y})}{\delta_S^2 (\Mat{X},\Mat{Y})} = 
	\lim_{t \rightarrow 0} \:\frac{\delta_g^2 \Big(\Mat{I}_d,\DIAG \big(\exp (t\Vec{\nu})\big)\Big)}
	{\delta_S^2 \Big(\Mat{I}_d,\DIAG \big(\exp (t\Vec{\nu})\big)\Big)}
	\notag \\
	&= \lim_{t \rightarrow 0} \:\frac
	{\Big\| \log \Big( \DIAG\big( \exp(t \Vec{\nu})\big) \Big)\Big\|_F^2}
	{\ln \Big|\frac{1}{2}\DIAG\big( 1 + \exp(t \Vec{\nu})\big)\Big|
	 -\frac{1}{2}\ln \Big|\DIAG\big( \exp(t \Vec{\nu})\big)  \Big|  }			
	\notag \\
	&= \lim_{t \rightarrow 0} \:\frac{t^2\sum_{i = 1}^d\limits \nu_i^2}
	{\sum_{i = 1}^d\limits \ln\Big(1 + \exp(t \nu_i)\Big) -t\sum_{i = 1}^d\limits \frac{\nu_i}{2} -d\ln(2)}			
	\label{eqn:proof_stein_airm_1} \\
	&= \lim_{t \rightarrow 0} \:\frac{2\sum_{i = 1}^d\limits \nu_i^2}
	{\sum_{i = 1}^d\limits{\frac{\nu_i^2\exp(t \nu_i)}{\big(1+\exp(t\nu_i)\big)^2}}}
	\label{eqn:proof_stein_airm_2} 			
	= 8\;,
	\end{align}
	where L'H\^{o}pital's rule was used twice from \eqref{eqn:proof_stein_airm_1} to \eqref{eqn:proof_stein_airm_2} since 
	the limit in \eqref{eqn:proof_stein_airm_1} is indefinite. 	
	Therefore,  
	\begin{equation*}
	\lim_{\Mat{X} \rightarrow \Mat{Y}} \:\frac{\delta_g (\Mat{X},\Mat{Y})}{\delta_S (\Mat{X},\Mat{Y})} =2\sqrt{2},
	\end{equation*}	
	which concludes the proof.
\end{proof}

\section{Conjugate Gradient on Grassmann Manifolds}
\label{sec:conjugate_gradient}

In our formulation, we model the projection $\Mat{W}$ as a point  
on a Grassmann manifold $\mathcal{G}(m,n)$. The Grassmann manifold $\mathcal{G}(m,n)$ consists of the set of all linear 
$m$-dimensional subspaces of $\mathbb{R}^n$. In particular, this lets us handle constraints of the form 
$\Mat{W}^T\Mat{W} = \Mat{I}_m$. Learning the projection then boils down to solving a non-linear optimization problem on the Grassmann manifold. Here, we employ a conjugate gradient (CG) method on the manifold,
which requires some notions of differential geometry reviewed below.

In differential geometry, the shortest path between two points on a manifold is a curve called a {\it geodesic}. 
The {\it tangent space} at a point on a manifold is a vector space that consists of the tangent vectors of all possible 
curves passing through this point.
Unlike flat spaces, on a manifold one cannot transport a tangent vector $\Delta$ from one 
point to another point by simple translation. To get a better intuition, take the case where the manifold is a sphere, and consider
two tangent spaces, one located at the pole and one at a point on the equator. Obviously the tangent vectors at the pole 
do not belong to the tangent space at the equator. Therefore, simple vector translation is not sufficient. 
As illustrated in Fig.~\ref{fig:parallel_transport}, transporting 
$\Delta$ from $\Mat{W}$ to $\Mat{V}$ on the manifold $\mathcal{M}$ requires subtracting the normal component 
$\Delta_{\perp}$ at $\Mat{V}$ for the resulting vector to be a tangent vector. Such a transfer of tangent vector 
is called {\it parallel transport}. Parallel transport is required by the CG method to compute the new descent direction  
by combining the gradient direction at the current and previous solutions.

\begin{figure}[!tb]
	\centering
	\includegraphics[width = 0.6\columnwidth]{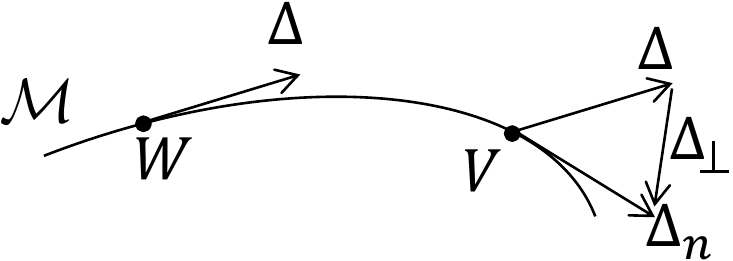}
	\caption{Parallel transport of a tangent vector $\Delta$ from a point $\Mat{W}$ to another point $\Mat{V}$ on the manifold.}
	\label{fig:parallel_transport}	
\end{figure}

On a Grassmann manifold, the above-mentioned operations have efficient numerical forms and can thus be used to 
perform optimization on the manifold. CG on a Grassmann manifold can be summarized by the following steps:
\begin{itemize}
	\item[\bf (i)] Compute the gradient $\nabla_\Mat{W} L$ of the objective function $L(\Mat{W})$ on the manifold at 
	the current solution using
	\begin{equation}
		\nabla_\Mat{W}L = D_\Mat{W}L - \Mat{W} \Mat{W}^T D_\Mat{W}L\;.
		\label{eqn:gradient_manif}
	\end{equation}
		
	\item[\bf (ii)] Determine the search direction $\Mat{H}$ by parallel transporting the previous search 
	direction and combining it with $\nabla_\Mat{W}L$.
		
	\item[\bf (iii)] Perform a line search along the geodesic at $\Mat{W}$ in the direction $\Mat{H}$.
	On the Grassmann manifold, the geodesics going
	from point $\Mat{X}$ in direction $\Delta$ can be represented
	by the Geodesic Equation~\cite{Absil_2008}
\begin{equation}
	\Mat{X}(t) =
	\begin{bmatrix}
		\Mat{X}\Mat{V} &\Mat{U}
	\end{bmatrix}
	\begin{bmatrix}
		\cos (\Sigma t) \\
		\sin (\Sigma t)
	\end{bmatrix}
	\Mat{V}^T
	\label{eqn:geodesic_curve}
\end{equation}
where $t$ is the parameter indicating the location along the geodesic, and $\Mat{U} \Sigma \Mat{V}^T$
is the compact singular value decomposition of $\Delta$. 
\end{itemize}
These steps are repeated until convergence to a local minimum, or until a maximum number of iterations is reached.

\section{Additional Experiments}
\label{sec:additional_exp}

\subsection{Parameter Sensitivity}
In all our experiments, the parameters of our approach were set in a principled manner (\ie, $\nu_w$ as the minimum number of samples in one class, and $\nu_b$ by cross-validation). 
In this section, we nonetheless study the influence of the number of nearest neighbor from different classes ($\nu_b$) on the overall performance.
To this end, we employed the UIUC material dataset and report the accuracy of our NN-Stein-ML method when varying this parameter and fixing the other to the value reported in Section~5 ($\nu_w = 6$).
Fig.~\ref{fig:nu_b_analysis} depicts the recognition accuracy for values of $\nu_b$ in the interval $[1,12]$. Note that for 
$\nu_b = 1$, which is equivalent to mainly considering the intra-class discrimination, the performance drops.  
For $\nu_b = 12$, which makes the inter-class discrimination dominant, the performance drops even further. 
The maximum performance of $58.6\%$ is reached for $\nu_b = 4$, which again shows that balance between the intra-class and inter-class terms is important. Note that our cross-validation procedure led to $\nu_b = 3$, which is not the optimal value on the test data, but still yields good accuracy.

\def \SENS_SIZE {0.85}
\begin{figure}[!t]		
   	\centering
   	\includegraphics[width= \SENS_SIZE \textwidth]{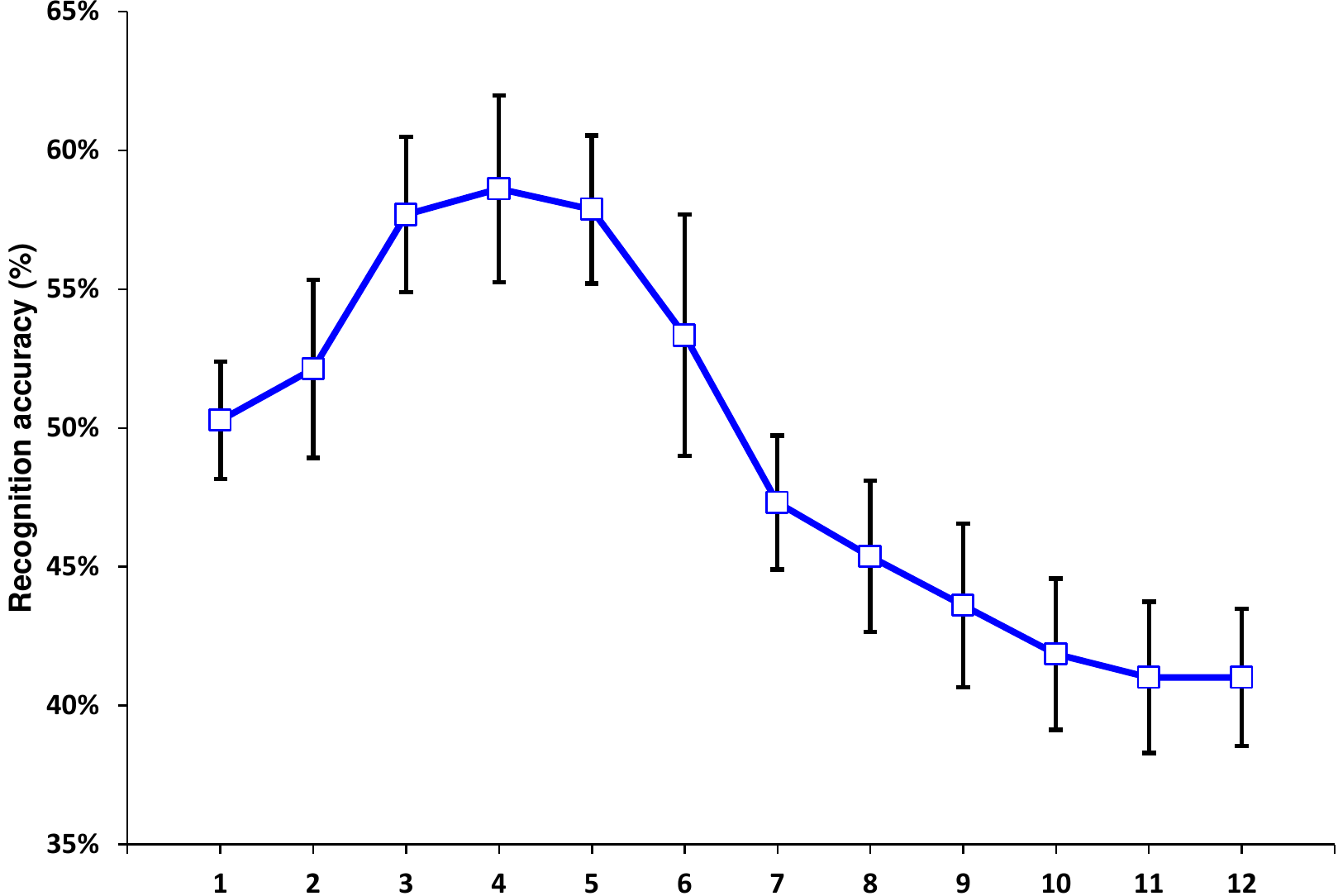}
	\caption{Accuracy on the UIUC material dataset for varying values of $\nu_b$.}
	\label{fig:nu_b_analysis}
\end{figure}

\subsection{Influence of the Number of Observations}
\label{sec:sensitivity}
Finally, as discussed in Section~\textcolor{red}{4.3}, we studied the sensitivity of our learning method to the number of observations used to build the RCMs.
To this end, we employed the UIUC material dataset. For the training images, where computational cost is unimportant, we generated RCMs using all possible observations (our setup provided
us with 9600 observations per image). For the test RCMs, we reduced the number of observations on an octave basis, \ie, downsampled the number of observations by a factor of two repetitively. 
\def \SEN_SIZE {0.85}
\begin{figure}[!t]	
	\centering
	\includegraphics[width = \SEN_SIZE \linewidth]{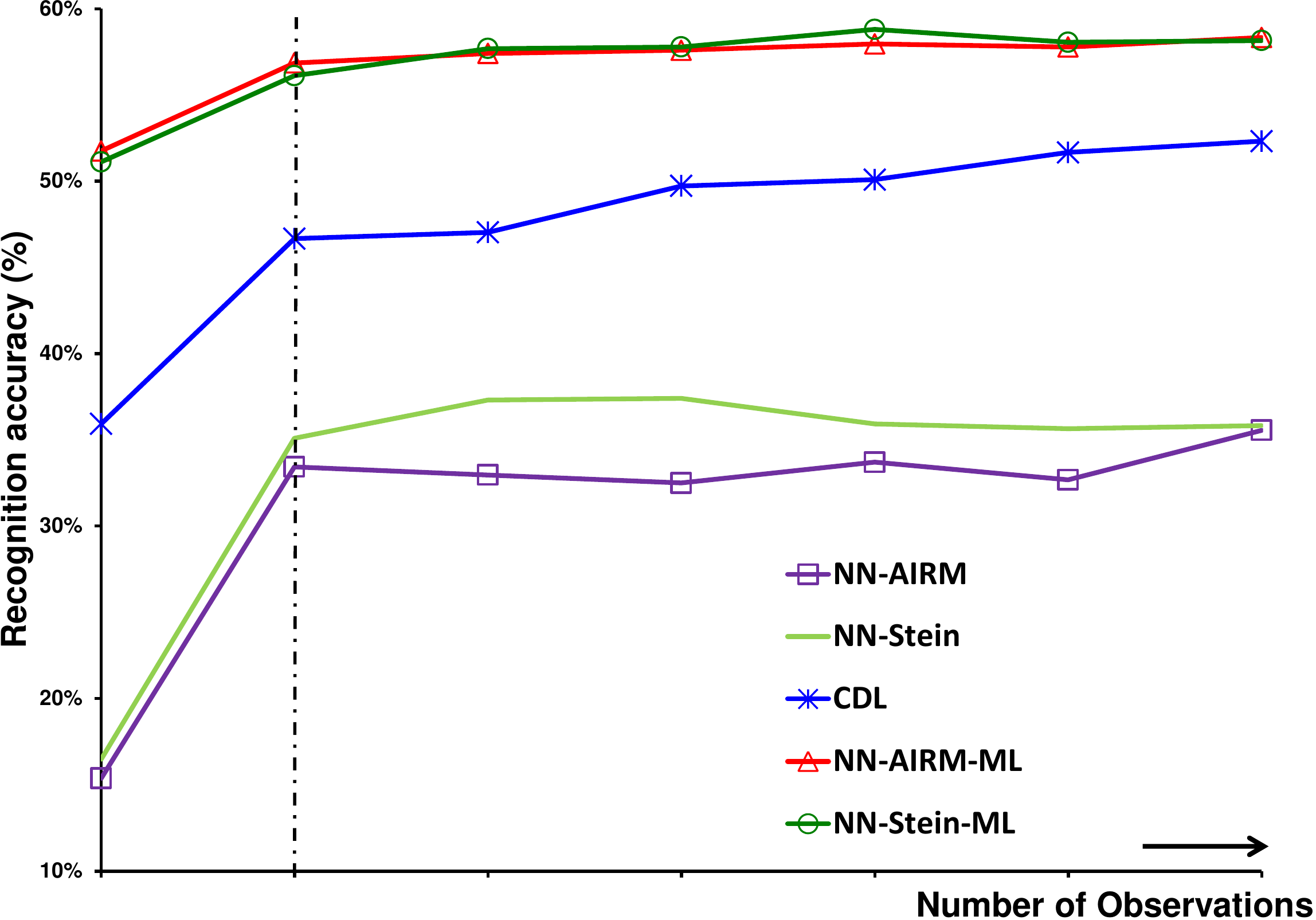}	
	\caption{\small Sensitivity of different algorithms to the number of observations used to create RCMs.}
	\label{fig:variability_analysis}
\end{figure}
Fig.~\ref{fig:variability_analysis} depicts the performance of CDL, as well as of NN classifiers with both the Stein metric and the AIRM, with and without our learning scheme. The point where the number of observations $r$ matches the size of the RCM $n$ (\ie, minimum number of observations to have a valid SPD matrix) is marked by a vertical dashed line. On the left side of this line, the number of observations is less than $n$. Therefore, for CDL, NN-Stein and NN-AIRM, a small regularizer of the form $\epsilon \mathbf{I}_n$ has to be added to the RCMs to make them positive definite. Note that no such regularizer was necessary when using our approach.
From Fig.~\ref{fig:variability_analysis}, we can see that all algorithms have a stable performance when the number of observations is large enough. When reducing the number of observations below $n$, the performance of CDL, NN-Stein and NN-AIRM drops down by 17\%, 19\% and 20\%, respectively. In contrast, with our learning algorithm, the drop in performance is less than 7\%. 

\end{document}